\documentclass[sigconf,natbib=false]{acmart}
\AtBeginDocument{%
  }

\copyrightyear{2026}
\acmYear{2026}
\setcopyright{cc}
\setcctype{by}
\acmConference[CIKM '26]{Proceedings of the 35th ACM International Conference on Information and Knowledge Management}{November 07--11, 2026}{Rome, Italy}
\acmBooktitle{Proceedings of the 35th ACM International Conference on Information and Knowledge Management (CIKM '26), November 07--11, 2026, Rome, Italy}
\acmDOI{10.1145/3799682.3841029}
\acmISBN{979-8-4007-2539-5/2026/11}

\RequirePackage[
  datamodel=acmdatamodel,
  style=acmnumeric,
  ]{biblatex}

\usepackage{algorithm}
\usepackage{algorithmic}

\usepackage{amssymb} 
\usepackage{multirow}
\usepackage{makecell}
\usepackage{hyperref}       
\usepackage{url}            
\usepackage{booktabs}       
\usepackage{amsfonts}       
\usepackage{nicefrac}       
\usepackage{microtype}      
\usepackage{array} 
\usepackage{amsthm}

\usepackage{amsmath}
\usepackage{enumitem}

\usepackage{tabularx}

\usepackage{thmtools}
\usepackage{graphicx} 
\usepackage{wrapfig}
\usepackage{todonotes}
\usepackage{stfloats} 
\usepackage[table]{xcolor}
\usepackage{colortbl}
\usepackage{comment}

\usepackage{xurl}
\definecolor{myblue}{RGB}{31,119,180}
\newcolumntype{C}{>{\columncolor{myblue!20}}c} 
\newcolumntype{B}{>{\columncolor{myblue}\color{white}}c}

\theoremstyle{definition} 

\newtheorem{theorem}{Theorem}[section]

\newtheorem{lemma}[theorem]{Lemma}
\newtheorem{corollary}[theorem]{Corollary}

\newtheorem{assumption}[theorem]{Assumption}

\newtheorem{rem}{Remark} 

\definecolor{tableZebra}{gray}{0.93}
\definecolor{myBlue}{RGB}{240, 245, 250}

\definecolor{linkLightBlue}{RGB}{0,102,204}

\begin{document}

\title{Kernel-Complexity Edge Sanitization for Training-Free Defense against Structural Graph Attacks}

\author{Yaning Jia}
\email{yaning.jia.gr@dartmouth.edu}
\affiliation{%
  \department{Department of Computer Science}
  \institution{Dartmouth College}
  \city{Hanover}
  \state{New Hampshire}
  \country{USA}
}

\author{Shenyang Deng}
\email{shenyang.deng.gr@dartmouth.edu}
\affiliation{%
  \department{Department of Computer Science}
  \institution{Dartmouth College}
  \city{Hanover}
  \state{New Hampshire}
  \country{USA}
}

\author{Yaoqing Yang}
\email{yaoqing.yang@dartmouth.edu}
\affiliation{%
  \department{Department of Computer Science}
  \institution{Dartmouth College}
  \city{Hanover}
  \state{New Hampshire}
  \country{USA}
}

\author{Chiyu Ma}
\email{chiyu.ma.gr@dartmouth.edu}
\affiliation{%
  \department{Department of Computer Science}
  \institution{Dartmouth College}
  \city{Hanover}
  \state{New Hampshire}
  \country{USA}
}

\author{Wenxuan Xu}
\email{wenxuan.xu.gr@dartmouth.edu}
\affiliation{%
  \department{Department of Computer Science}
  \institution{Dartmouth College}
  \city{Hanover}
  \state{New Hampshire}
  \country{USA}
}

\author{Soroush Vosoughi}
\correspondingauthor
\email{soroush.vosoughi@dartmouth.edu}
\affiliation{%
  \department{Department of Computer Science}
  \institution{Dartmouth College}
  \city{Hanover}
  \state{New Hampshire}
  \country{USA}
}

\renewcommand{\shortauthors}{Yaning Jia et al.}

\begin{abstract}
Graph Neural Networks (GNNs) have achieved remarkable success across diverse applications, yet they remain highly vulnerable to adversarial attacks that maliciously perturb graph structure. Existing defenses often lack rigorous theoretical grounding, rely on attack-specific heuristics, or require costly retraining procedures such as adversarial training. To address these limitations, we propose Kernel-Complexity Edge Sanitization (KCES), a training-free and model-agnostic framework for defending against structural attacks. KCES is built upon Graph Kernel Complexity (GKC), a principled metric derived from the graph Gram matrix that appears in a generalization upper bound on the GNN test error. From this bound, we define an edge-specific KC score that quantifies each edge’s structural influence via its induced change in GKC. KCES then identifies and prunes high-KC edges, which are empirically enriched with adversarial perturbations under structural attacks, to mitigate their harmful impact. Computationally efficient and scalable, KCES operates as a lightweight preprocessing step without retraining and can be seamlessly integrated with existing defenses. Extensive experiments demonstrate that KCES consistently outperforms representative robust baselines across diverse attack settings and scales effectively to large graphs. Supported by theoretical analysis and extensive empirical validation, KCES provides a principled and efficient framework for securing~GNNs. Our code is available at {\color{linkLightBlue}\url{https://github.com/karpning/KCScore}}.
\end{abstract}

\ccsdesc[300]{Information systems~Data mining}
\ccsdesc[500]{Theory of computation~Theory and algorithms for application domains}
\ccsdesc[500]{Computing methodologies~Artificial intelligence}

\keywords{Graph Neural Networks, Graph Adversarial Defense, Kernel Complexity}

\maketitle

\section{Introduction}
\label{intro}

\label{sec:intro}
Graph Neural Networks (GNNs) have achieved remarkable success in modeling graph-structured data across diverse domains, including social analysis \cite{perozzi2014deepwalk, grover2016node2vec}, recommendation~\cite{ying2018graph, he2020lightgcn}, and drug discovery \cite{gilmer2017neural, yang2019analyzing}. Despite their effectiveness, GNNs are highly vulnerable to structural perturbations: even the addition or removal of a few edges can drastically degrade performance. Such attacks exploit the message-passing mechanism by injecting spurious connections or disrupting informative neighborhoods, thereby corrupting the aggregated information. Early studies, including Nettack and Metattack~\cite{zugner2018adversarial,xu2018powerful}, first revealed this weakness, and subsequent topology-only attacks further amplified the threat~\cite{bojchevski2019adversarial,geisler2021robustness, xu2019topology, wang2022bandits}. As these attacks require no feature manipulation and operate under limited perturbation budgets, purely structural attacks represent a realistic and severe threat surface for deployed GNN systems.


To mitigate structural vulnerabilities, a variety of defense strategies have been proposed. Graph purification methods, such as GNN-Jaccard~\cite{wu2019adversarial}, GNN-SVD~\cite{entezari2020all}, and NoiseGNN~\cite{ennadir2024simple}, aim to remove suspicious edges to restore a cleaner topology. Adversarial training approaches, including RGCN~\cite{zhu2019robust}, improve robustness by training on perturbed graphs. Architecture- or reconstruction-based defenses, such as ProGNN~\cite{jin2020graph}, AirGNN~\cite{liu2021graph}, and GPR-GAE~\cite{lee2025self}, incorporate adaptive graph learning or joint structure refinement into the model design. Beyond these approaches, recent work has also explored theoretically grounded stability mechanisms based on Lipschitz analysis and regularization to control GNN output sensitivity under perturbations and biased inputs~\cite{jia2023enhancing,jia2024aligning,jia2023stabilizing}. Despite promising empirical performance, existing defenses exhibit several limitations. Many defenses are motivated by local structural heuristics, such as feature similarity, low-rank smoothness, or learned edge weights, but their pruning criteria are not explicitly tied to a generalization objective. This makes it difficult to interpret why a removed edge should improve the expected behavior of downstream GNNs. Optimization-based methods often introduce substantial computational overhead, limiting their applicability to large-scale graphs, while training-based defenses can become specialized to particular perturbation patterns such as Nettack~\cite{zugner2018adversarial}. These limitations highlight the need for a theoretically grounded, training-free approach that is both robust and scalable across diverse architectures.

\begin{figure}[t]
  \centering
  \includegraphics[width=1.0\linewidth]{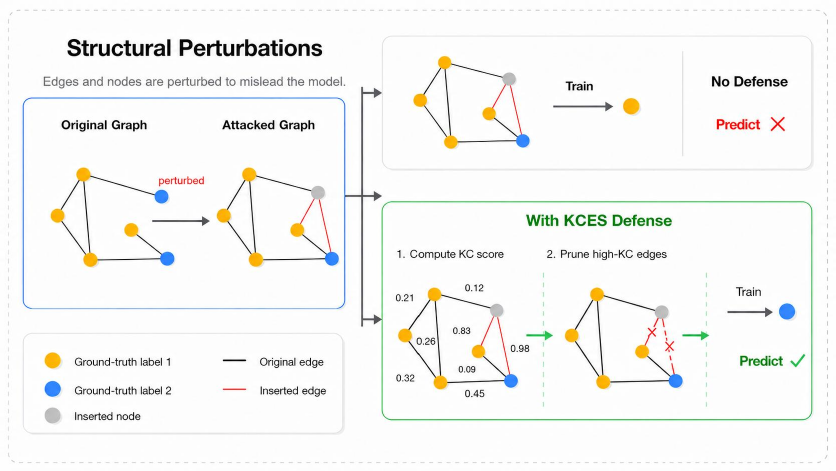}
  \vspace{-0.5em}
  \caption{Defense Framework against Structural Perturbations with KCES. Adversarial attacks tend to increase the KC scores of perturbed edges; KCES enhances robustness by pruning edges with high KC scores.}
  \label{fig-1:kces_illustration}
\end{figure}

Motivated by these limitations and inspired by advances in data-independent generalization~\cite{arora2019fine,nohyun2022data}, we propose \textbf{Kernel-Complexity Edge Sanitization (KCES)}, a principled, training-free, and model-agnostic framework for defending against structural perturbations. KCES introduces \textbf{Graph Kernel Complexity (GKC)} (Section~\ref{sub-GKC}), a graph-induced complexity metric that appears in the test error bound of GNNs. Building on this formulation, we define the \textbf{KC score} (Section~\ref{KCscoredef}) of an edge as the change in GKC induced by its removal. This score measures the magnitude of an edge's influence on the GKC term in the generalization bound, yielding a generalization-guided pruning signal beyond local similarity or reconstruction heuristics. Since KC scores depend only on the graph-feature pair and pseudo-labels, KCES requires no retraining or architectural modification and can be integrated with diverse GNN architectures. As illustrated in Figure~\ref{fig-1:kces_illustration} and validated experimentally (Section~\ref{subsec:mechanism}), structural attacks enrich harmful perturbations in the high-KC region; by selectively pruning these high-KC edges, KCES improves GNN robustness in a scalable and training-free manner. Our contributions are summarized as follows:
\begin{itemize}[leftmargin=0.8em]
\item \textbf{A generalization-theoretic framework for structural defense.}
We introduce Graph Kernel Complexity (GKC), a graph-induced complexity measure that appears in a GNN generalization bound, and derive the KC score as a generalization-guided measure of each edge's structural influence.

\item \textbf{A training-free edge sanitization mechanism.}
We propose Kernel-Complexity Edge Sanitization (KCES), a preprocessing strategy that removes high-KC edges empirically enriched with harmful perturbations under structural attacks, without adversarial training or iterative graph optimization.

\item \textbf{Plug-and-play and scalable design.}
KCES functions as a model-agnostic preprocessing module that can be seamlessly integrated with diverse GNN architectures and existing defense pipelines. Since KC scores depend only on the graph-feature pair and pseudo-labels, KCES requires no retraining or architectural modification and remains efficient on large-scale graphs.

\item \textbf{Comprehensive empirical validation.}
Extensive experiments across multiple attacks and graph scales show that KCES consistently improves structural robustness over representative robust baselines.
\end{itemize}

\vspace{-0.5em}
\section{Preliminaries}
\label{sec:pre}

We denote $\mathbb{R}^d$ as the $d$-dimensional Euclidean space. The norm $\|\cdot\|_p$ represents the vector $p$-norm and its induced operator norm for matrices, while $\|\cdot\|_F$ denotes the Frobenius norm. The notation $\|\cdot\|_{V_L,p}$ refers to the $p$-operator norm restricted to a subspace $V_L \subset \mathbb{R}^d$. The shorthand $[n]=\{1,2,\ldots,n\}$ denotes the index set. For a distribution $\mathcal{D}$, $\mathbb{E}_{X\sim\mathcal{D}}[\cdot]$ denotes expectation with respect to $X$. We represent a graph as $G=(X,\tilde A,\tilde D,\mathbf y)$, where $X\in\mathbb{R}^{N\times F}$ is the node feature matrix, $\tilde A\in\mathbb{R}^{N\times N}$ is the adjacency matrix with self-loops, $\tilde D_{ii}=\sum_j \tilde A_{ij}$ is the degree matrix, and $\mathbf y\in\mathcal Y^N$ denotes the label vector. The aggregated feature matrix is $Z=\tilde D^{-1/2}\tilde A \tilde D^{-1/2} X$, and its row-normalized version is denoted by $\tilde X$ in the Gram matrix. The training subgraph is denoted by $G_{\mathrm{train}}$. The first-layer weight matrix is $W\in\mathbb{R}^{F\times m}$ with columns $W_r\in\mathbb{R}^F$, and the second-layer weight vector is $\mathbf a\in\mathbb{R}^m$. For a matrix $K\in\mathbb{R}^{N\times N}$, $K_{ij}$ denotes its $(i,j)$-th entry. A Gaussian distribution with mean $\mu$ and covariance $\Sigma$ is written as $\mathcal N(\mu,\Sigma)$.

\subsection{Theoretical Setup for GNNs}
\label{gnn-definition}
We analyze a GNN as a kernel model on an undirected graph with \(N\) nodes. The adjacency matrix $\tilde A\in\{0,1\}^{N\times N}$ is fixed, symmetric, and includes self-loops ($\tilde A_{ii}=1$), with degree matrix $\tilde D_{ii}=\sum_j \tilde A_{ij}$. The data distribution $\mathcal D_G$ over $\mathbb R^{N\times F}\times \mathbb R^N$ generates node features and labels $G=(X,\tilde A,\tilde D,\mathbf y)$, where each feature-label pair $(X_i,y_i)\in\mathbb R^F\times\mathbb R$ is drawn independently for analytical tractability. A subset of nodes and labels, denoted by $G_{\mathrm{train}}$, is used for training. Let $Z=\tilde D^{-\frac12}\tilde A\tilde D^{-\frac12}X$ denote the aggregated feature matrix, and let $\tilde X$ denote its row-normalized version. The forward propagation for node $i$ is
\vspace{-0.5em}
\begin{equation}
\label{GCN}
f_{\text{GNN}}(X_i, \tilde{A}, \tilde{D}) = \frac{1}{\sqrt{m}}
\sum_{r=1}^m
a_r \,
\sigma\!\left(W_r^\top \tilde X_i\right).
\end{equation}
Here, $m$ is the number of hidden units; $W\in\mathbb R^{F\times m}$ is the first-layer weight matrix with columns $W_r\in\mathbb R^F$; $\mathbf a=(a_1,\dots,a_m)^\top\in\mathbb R^m$ is the second-layer weight vector; and $\sigma(\cdot)$ denotes the activation function (e.g., ReLU).

For the GNN \(f_{\text{GNN}}(X_i, \tilde{A}, \tilde{D})\), we define the training error (empirical risk) on \(G_{\text{train}}\) as:
\begin{equation}
\label{eq-trainloss}
L(W) = \frac{1}{2} \sum_{i=1}^{N} \left( y_i - f_{\text{GNN}}(X_i, \tilde{A}, \tilde{D}) \right)^2,
\end{equation}
Here, $X_i$ and $y_i$ are sampled from the training subset $G_{\mathrm{train}}$. 
The corresponding test error (expected risk) is defined as the expectation of the empirical risk over the data distribution~$\mathcal D_G$:
\begin{equation}
\label{eq-testloss}
L_{\mathcal D_G}(W)
=
\mathbb{E}_{(X,\mathbf y)\sim \mathcal D_G}
\big[ L(W) \big].
\end{equation}

\subsection{Graph Kernel Gram Matrix}
\label{gram-matrix}

We introduce the \textbf{Graph Kernel Gram Matrix} (hereafter, the \textbf{Gram matrix}), which characterizes pairwise node interactions induced by graph structure and node features. This construction builds upon classical kernel methods~\cite{xie2017diverse, arora2019fine, tsuchida2018invariance, du2018gradient}, where Gram matrices serve as a central tool for relating model complexity to generalization performance. In our GNN setting, \textit{the graph aggregation operation is incorporated directly into the kernel}. Specifically, for a graph \(G=(X,\tilde A,\tilde D,\mathbf y)\), let \(Z=\tilde D^{-\frac12}\tilde A\tilde D^{-\frac12}X\in\mathbb R^{N\times F}\) denote the aggregated feature matrix, and define the row-normalized features \(\tilde X_i=Z_i/\|Z_i\|_2\) before kernel computation so that \(\|\tilde X_i\|_2=1\) for each node \(i\). We then define the Gram matrix \(H^\infty=[H_{ij}^\infty]_{i,j=1}^N\in\mathbb R^{N\times N}\), where each entry \(H_{ij}^\infty\) is given by:
\vspace{-0.1em}
\begin{align}
\label{gkgm}
    H_{ij}^{\infty} = \frac{\tilde{X}_i^{\top} \tilde{X}_j \left(\pi - \arccos \left(\tilde{X}_i^{\top} \tilde{X}_j\right)\right)}{2 \pi}.
\end{align}
Here, \(\tilde X_i \in \mathbb{R}^{1\times F}\) and \(\tilde X_j \in \mathbb{R}^{1\times F}\) denote the \(i\)-th and \(j\)-th rows of \(\tilde X\), respectively. The resulting matrix \(H^\infty\) induces a kernel feature space over graph nodes, which serves as the foundation for defining the Graph Kernel Complexity in the next section.

\subsection{Graph Kernel Complexity}
\label{sub-GKC}
Based on the above Gram matrix, we define the \textbf{Graph Kernel Complexity (GKC)} to characterize the test error bound of GNNs, reflecting their generalization capacity. Formally, given the Gram matrix \(H^\infty \in \mathbb{R}^{N \times N}\) and a bounded label vector \(\mathbf y \in \mathbb{R}^N\), the GKC is defined as
\begin{equation}
\label{eq-GKC}
    \mathrm{GKC}(H^\infty,\mathbf y)
    =
    \frac{2\mathbf y^\top (H^\infty)^{-1}\mathbf y}{N}.
\end{equation}
In the theoretical analysis, \(\mathbf y\) denotes the ground-truth label vector. In KCES, we replace \(\mathbf y\) with the pseudo-label vector \(\hat{\mathbf y}\), obtained by mapping unsupervised clustering assignments to bounded numerical codes.

\section{GKC-based Generalization Analysis}
\label{sec:gkc_analysis}
This section employs the Gram matrix and GKC to establish generalization bounds for the GNN defined in Section~\ref{gnn-definition}. Theorems for training and test errors are presented informally for clarity, with formal statements and proofs in the Appendix~\ref{sec:theory-gkm}.

First, Theorem~\ref{thm:train-error} characterizes the training error dynamics as follows:
\begin{theorem}[\textbf{Training error dynamics}]
\label{thm:train-error}
Under Assumption~\ref{assum-e-2}, after $t$ gradient descent updates with step size $\eta$, the training error satisfies
\begin{equation}
\label{thm:train-error-eq}
L(W_t) = \left\| \left( I - \eta H^\infty \right)^t \mathbf{y} \right\|_2^2 + \varepsilon,
\end{equation}
where $H^\infty$ denotes the Gram matrix, $m$ is the hidden layer width, and $\varepsilon = \tilde{O}(m^{-1/2})$ represents an error term dependent on $t$ and $m$. The formula shows that the training error decays exponentially with the number of gradient steps, with the rate of decay governed by the Gram matrix $H^\infty$. This behavior underlies the generalization bound stated in Theorem~\ref{thm:test-error}. A formal version of the result is provided in Appendix~\ref{subsec:proof-training}~(Theorem~\ref{train_theo_fo}).
\end{theorem}

Building on the training error analysis, Theorem~\ref{thm:test-error} establishes a generalization bound in terms of GKC. 

\begin{theorem}[\textbf{Test error bound}]
\label{thm:test-error}
Under Assumption~\ref{assum-e-2}, for sufficiently large hidden layer width $m$ and iteration count $t$, the following holds with probability at least $1 - \delta$:
\begin{equation}
\label{thm:test-error-eq}
L_{\mathcal{D}_G}(W_t)
\;\le\;
\sqrt{\mathrm{GKC}(H^{\infty}, \mathbf y)}
\;+\;
O\left(\sqrt{\frac{\log \frac{N}{\lambda_0 \delta}}{N}}\right),
\end{equation}
Here, $\mathrm{GKC}(H^{\infty}, \mathbf y)$ denotes the theoretical Graph Kernel Complexity computed with the ground-truth label vector $\mathbf y$ (Section~\ref{sub-GKC}), $\delta \in (0, 1)$ is the confidence level, $N$ is the number of nodes, and $\lambda_0$ is a lower bound on $\lambda_{\min}(H^{\infty})$ under Assumption~\ref{assum-e-2}~(Appendix~\ref{subsec:theory-assumptions}). The generalization bound on the GNN test error is primarily influenced by the data-dependent GKC term. A smaller GKC leads to a tighter bound, indicating improved generalization under standard GNN training regimes. The formal statement and proof are presented in Appendix~\ref{subsec:proof-test} (Theorem~\ref{test_theo_fo}).
\end{theorem}

\section{Kernel-Complexity Edge Sanitization for
Structural Robustness}
\label{sec:kces}

Motivated by the connection between GKC and GNN test error, we propose \textbf{Kernel-Complexity Edge Sanitization} (\textbf{KCES}), a training-free and model-agnostic framework that evaluates the structural influence of graph edges and prunes high-KC edges empirically enriched with harmful perturbations under structural attacks. This framework comprises three key components: \textbf{\textit{(i) Pseudo-Label Generation}}, \textbf{\textit{(ii) Edge KC Score Estimation}}, and \textbf{\textit{(iii) Kernel-Complexity Edge Sanitization}}.

\subsection{Pseudo-Label Generation}
\label{sec:pseudo-generation}
To preserve the training-free and unsupervised nature of KCES, the label vector required by the GKC metric (Eq.~\ref{eq-GKC}) is generated via K-Means~\cite{macqueen1967some}, applied to the row-normalized aggregated node representations:
\begin{equation}
    \hat{\mathbf y} = \operatorname{K\text{-}Means}\!\left(\tilde X,\, k\right).
\end{equation}
Here, $k$ is set to the number of classes, and \(\tilde X\) denotes the row-normalized aggregated features defined in Section~\ref{gram-matrix}. The clustering assignments are mapped to bounded numerical codes before computing GKC. Rather than approximating ground-truth semantics, these pseudo-labels are intended to capture structural smoothness over the kernel-induced geometry, enabling KCES to identify edges that violate local smoothness, such as adversarial inter-community connections.

\subsection{Edge KC Score Estimation}
\label{KCscoredef}
With pseudo-labels providing the basis for GKC computation, we quantify the structural contribution of each edge via the \textbf{KC score}. For an edge $e_{ij}$ connecting nodes $i$ and $j$, we define
\begin{equation}
\label{kc_def}
  KC(i, j) 
  = 
  \left| 
  \mathrm{GKC}\!\left(H^{\infty}, \hat{\mathbf y}\right) 
  - 
  \mathrm{GKC}\!\left(H^{\infty}_{-(i,j)}, \hat{\mathbf y}\right) 
  \right|,
\end{equation}
where $H^{\infty}$ is the Gram matrix of the original graph, $H^{\infty}_{-(i,j)}$ is the Gram matrix of the modified graph $G_{-(i,j)}$ obtained by removing edge $e_{ij}$, and $\hat{\mathbf y}$ denotes the bounded pseudo-label vector generated in Section~\ref{sec:pseudo-generation}. The KC score, \(KC(i,j)\), quantifies the magnitude of the change in pseudo-label-based GKC caused by removing edge \(e_{ij}\). Since it is defined as an absolute difference, KC is direction-agnostic: a large value indicates strong structural influence on the graph-induced kernel complexity.

\begin{corollary}[Edge-specific test error bound]
\label{edge_theo}
Let $G_{-(i,j)}$ be the graph obtained by deleting edge $e_{ij}$, and let $W_t$ denote the GNN parameters after $t$ gradient descent updates on the modified graph. Under the assumptions of Theorem~\ref{thm:test-error}, and for sufficiently large hidden width $m$ and iteration count $t$, the following holds with probability at least $1-\delta$:
\begin{equation}
L_{\mathcal{D}_{G_{-(i,j)}}}(W_t)
\le
\sqrt{\mathrm{GKC}(H^\infty,\mathbf y)}
+
\sqrt{KC(i,j)}
+
\epsilon_N,
\end{equation}
where $\epsilon_N = O\!\left(\sqrt{\frac{\log (N / \lambda_0\delta)}{N}}\right)$. 
Here, $N$ denotes the number of nodes, $\lambda_0$ is a lower bound on $\lambda_{\min}(H^{\infty})$ (Assumption~\ref{assum-e-2}), $KC(i,j)$ denotes the edge KC score computed with the ground-truth label vector $\mathbf y$ in this theoretical bound, and $\mathcal{D}_{G_{-(i,j)}}$ is the data distribution induced by removing edge $e_{ij}$ from graphs sampled from $\mathcal D_G$. The detailed proof is provided in Appendix~\ref{subsec:proof-edge-bound}.
\end{corollary}

\paragraph{Interpretation}
Corollary~\ref{edge_theo} connects the KC score of an edge $e_{ij}$ to the expected test error of the modified graph $G_{-(i,j)}$, thereby quantifying its structural influence. Edges with large KC values have strong influence on graph-induced kernel complexity. While KC itself does not encode the signed direction of this influence, structural attacks tend to introduce high-influence edges that disrupt local smoothness. This behavior is empirically examined in Section~\ref{subsec:mechanism}. Although the corollary uses the ground-truth-label version of $KC(i,j)$, practical KCES computes the pseudo-label-based score in Eq.~\ref{kc_def} as a surrogate ranking signal. We do not require semantic equivalence; instead, pseudo-labels capture structural partitions of the graph. Since adversarial perturbations primarily violate structural smoothness, pseudo-label-based KC scores can provide a practical ranking signal for edge pruning.

\subsection{Kernel-Complexity Edge Sanitization}
\label{sec:kces-algorithm}
Building on estimated KC scores, KCES sanitizes attacked graphs by removing high-influence edges. Although KC scores are direction-agnostic, structural attacks tend to introduce high-KC perturbations that disrupt local smoothness, as empirically validated in Section~\ref{subsec:mechanism}. The specific \textbf{KCES} procedure is described in Algorithm~\ref{alg:kces}.

\begin{algorithm}[h]
\caption{KCES: Kernel-Complexity Edge Sanitization}
\label{alg:kces}
\begin{flushleft}
\textbf{Input:} Graph $G=(V,E)$; row-normalized features $\tilde X$; pruning ratio $\alpha\in[0,1]$ \\
\textbf{Output:} Sanitized graph $G'=(V,E')$
\end{flushleft}
\begin{algorithmic}[1]

\STATE $\hat{\mathbf y} \leftarrow \operatorname{KMeans}\!\left(\tilde X,\, k\right)$
\hfill \small{\textit{// Pseudo labels}}

\STATE $H^{\infty} \leftarrow \operatorname{GKGM}(\tilde X)$
\hfill \small{\textit{// Kernel Gram matrix}}
\vspace{0.3em}

\STATE $\text{GKC}_0 \leftarrow \operatorname{ComputeGKC}(H^{\infty}, \hat{\mathbf y})$
\hfill \small{\textit{// Baseline complexity}}

\vspace{0.4em}
\FOR{each edge $e_{ij}\in E$}
\vspace{0.3em}
    \STATE $G_{-(i,j)} \leftarrow G \setminus \{e_{ij}\}$
    \hfill \small{\textit{// Remove edge}}
    \vspace{0.3em}

    \STATE $H^{\infty}_{-(i,j)} \leftarrow \operatorname{GKGM}(G_{-(i,j)}, \tilde X)$
    \hfill \small{\textit{// Recompute kernel}}
    \vspace{0.3em}

    \STATE $\text{KC}[e_{ij}] \leftarrow 
    \bigl|\, \text{GKC}_0 -
    \operatorname{ComputeGKC}(H^{\infty}_{-(i,j)},\hat{\mathbf y}) \,\bigr|$
    \hfill \small{\textit{// Compute KC}}
    \vspace{0.3em}
\ENDFOR

\vspace{0.4em}
\STATE $(e_{(1)}, e_{(2)}, \dots, e_{(|E|)}) \leftarrow \operatorname{Sort}(\operatorname{KC}, \text{descending})$
\hfill \small{\textit{// Sort edges by KC}}
\vspace{0.3em}

\STATE $k \leftarrow \lceil \alpha |E| \rceil$
\hfill \small{\textit{// Pruning size}}
\vspace{0.3em}

\STATE $E' \leftarrow E \setminus \{e_{(1)}, \dots, e_{(k)}\}$
\hfill \small{\textit{// Remove top-$k$ edges}}

\STATE \textbf{Return} $G'=(V,E')$

\end{algorithmic}
\end{algorithm}

\paragraph{Applicability of KCES.}
Although our theoretical framework is derived from the analysis of a specific GNN hypothesis space, the resulting KCES framework is broadly applicable across diverse GNN architectures. This versatility stems from its core metric, GKC, which is fundamentally model-agnostic in its computation. This principle is best understood through an analogy with the data covariance matrix; while formally linked to linear models, properties of the covariance matrix like its trace and rank~\cite{bartlett2020benign} serve as universal tools for diagnosing data complexity even in deep learning models~\cite{raghu2017svcca}. Similarly, GKC functions as a specialized diagnostic tool for the graph learning domain. By embedding the graph aggregation mechanism directly into its kernel, GKC provides a diagnostic measure of graph-induced kernel geometry, rather than a purely topological quantity, while remaining independent of any specific GNN parameters. This allows KCES to function as a versatile, plug-and-play defense module capable of enhancing the structural robustness of a wide variety of GNNs, not just those with simple convolutional structures. Our experiments in Section~\ref{defense-performance} empirically confirm the broad applicability of KCES, demonstrating strong robustness across diverse GNN architectures such as GAT, RGCN, AirGNN, and ProGNN.


\paragraph{Computational Complexity of KCES.}
In the worst case, KCES recomputes the graph kernel Gram matrix and its inverse for each candidate edge, leading to time complexity $O(R E (N^2F+M^2F+M^3))$ and space complexity $O(N^2+NF+M^2)$, where $N$, $E$, $F$, $M$, and $R$ denote the number of nodes, edges, feature dimension, sampled node budget up to a constant class factor, and Monte Carlo repetitions, respectively. This worst-case bound is pessimistic. In practice, KCES uses sparse graph operations, local $h$-hop subgraphs of size $L\ll N$, and incremental inverse updates, reducing the practical time and space complexities to $O\!\left(R(EF+E(L^2F+L^2))\right)$ and $O(E+NF+L^2)$, respectively. Edge-wise KC score computations are independent and can be parallelized across GPUs or CPUs to further reduce wall-clock time. When $F$ is fixed or reduced by preprocessing, the practical time complexity simplifies to $O\!\left(R(EF+EL^2)\right)$. As KCES is applied as a one-time preprocessing step, its cost is amortized over subsequent training. The empirical evaluation is reported in Table~\ref{tab:runtime_space}, with detailed complexity analysis provided in the supplementary material (Section~3).

\section{Experiments}
\label{sec:exp}

This section evaluates KCES from four aspects: the mechanism of KC scores under structural attacks, robustness against representative baselines, scalability and computational efficiency, and sensitivity to pseudo-label generation.

\subsection{Overall Setup}
\label{setup}
\paragraph{Datasets} We evaluate the robustness and scalability of KCES on benchmark datasets spanning different graph scales: \textbf{(1) Small-scale graphs}: \textit{Cora}~\cite{sen2008collective}, \textit{Citeseer}~\cite{sen2008collective}, and \textit{Polblogs}~\cite{adamic2005political}; \textbf{(2) Medium-scale graphs}: \textit{Pubmed}~\cite{yang2016revisiting}; \textbf{(3) Large-scale graphs}: \textit{Flickr}~\cite{liu2009social}; and \textbf{(4) Massive-scale graphs}: \textit{Ogbn-Arxiv}~\cite{hu2020open}. Detailed dataset statistics, train/validation/test splits, and a detailed complexity analysis of KCES are provided in the supplementary material available at \url{https://github.com/karpning/KCScore/blob/main/supplementary_material.pdf}.

\paragraph{Structural Attack Strategies} 
To rigorously assess the structural robustness conferred by KCES, we evaluate against a suite of adversarial attacks that maliciously perturb the graph topology prior to model training. We employ representative methods spanning three categories: 
\textbf{(i) Non-targeted attacks}, which aim to degrade overall model performance by modifying the global graph structure, including \textit{Metattack}~\cite{zügner2018adversarial}, \textit{MINMAX}~\cite{xu2019topology}, and \textit{DICE}~\cite{waniek2018hiding}; 
\textbf{(ii) Targeted attacks}, which focus on misleading predictions for specific victim nodes, adopting the widely-used \textit{Nettack}~\cite{zugner2018adversarial}; and 
\textbf{(iii) Random attacks}, which simulate structural noise via random edge injections and deletions. Furthermore, for the massive-scale evaluation on \textit{Ogbn-Arxiv}, we explicitly employ \textit{PRBCD}~\cite{geisler2021robustness}, a scalable projected gradient-based attack to evaluate robustness on large-scale graphs (e.g., \textit{Ogbn-Arxiv}), where traditional attacks become computationally prohibitive.

\paragraph{Defense Baselines}
We compare KCES with a diverse set of baselines grouped into four categories. 
\textbf{(i) Vanilla GNNs:} Graph Convolutional Networks (GCN)~\cite{kipf2016semi} and Graph Attention Networks (GAT)~\cite{velivckovic2018graph}, which serve as standard architectures to quantify performance degradation under attacks. 
\textbf{(ii) Robust-by-design models:} Robust GCN (RGCN)~\cite{zhu2019robust}, AirGNN (AirG)~\cite{liu2021graph}, NoiseGNN (NoiseG)~\cite{ennadir2024simple}, and GNNGuard (G-Guard)~\cite{zhang2020gnnguard}, which incorporate defense mechanisms directly into the model architecture. 
\textbf{(iii) Graph purification methods:} GNN-Jaccard (G-Jac)~\cite{wu2019adversarial} and GNN-SVD (G-SVD)~\cite{entezari2020all}, which pre-process the graph structure to mitigate adversarial perturbations. 
\textbf{(iv) Graph optimization-based defenses:} ProGNN (Pro-G)~\cite{jin2020graph} and GPR-GAE (GPR)~\cite{lee2025self}, which jointly refine or reconstruct the graph structure during the model training process.

\paragraph{Supplementary Material} We provide a supplementary PDF to support reproducibility and completeness. It includes dataset statistics and splits (Section~2), detailed KCES complexity analysis (Section~3), and extended experimental analyses (Section~4). In particular, Section~4.2 reports additional structural-attack results under \textit{PRBCD} and \textit{LRBCD}~\cite{geisler2021robustness}; Section~4.4 provides feature-perturbation evaluations that clarify the scope of KCES; Section~4.5 reports runtime and memory analyses for KC score computation across graph scales; and Section~4.6 studies pruning-ratio ablations. The supplementary PDF is available at \url{https://github.com/karpning/KCScore/blob/main/supplementary_material.pdf}.

\subsection{Mechanism of KC Scores under Structural Graph Attacks} 
\label{subsec:mechanism}

This section examines whether KC scores capture harmful structural perturbations, as suggested by Corollary~\ref{edge_theo}. Since KC scores measure the magnitude of an edge's influence on GKC rather than the signed direction of this influence, high-KC edges are not necessarily harmful in clean graphs. However, under structural attacks, adversarial perturbations tend to create high-influence edges that disrupt kernel-label alignment. We therefore study whether structural attacks shift KC scores toward larger values and whether removing high-KC edges restores robustness. We conduct two complementary studies on \textit{Cora} and \textit{Pubmed}: first, we compare KC-score distributions on clean graphs, \textit{Metattack}-perturbed graphs, and graphs after high-KC edge pruning; second, we evaluate whether KC scores provide an effective pruning signal by comparing \textbf{High-KC Pruning}, \textbf{Low-KC Pruning}, and \textbf{Random Pruning} on attacked graphs. The first removes edges with the largest KC scores, the second removes edges with the smallest KC scores, and the third serves as a baseline. We vary the pruning ratio from 0.00 to 0.95 in increments of 0.05. Results on clean graphs are provided in the supplementary material (Section~4.7).

\begin{figure}[htbp]
  \centering
  \includegraphics[width=1.0\linewidth]{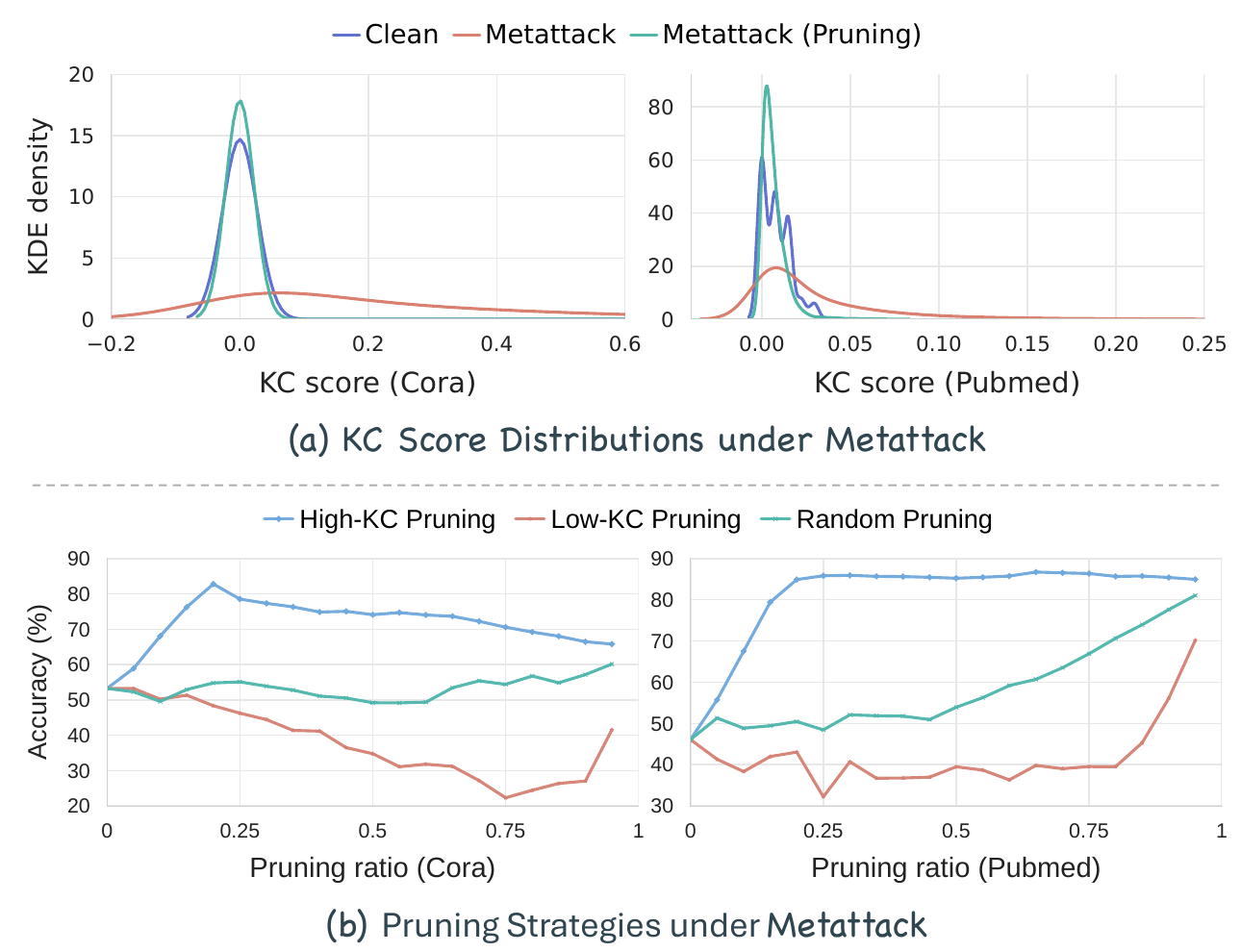}
  \vspace{-2em}
  \caption{Mechanism of KC scores under structural attacks.
  (a) KC-score distributions on clean, \textit{Metattack}-perturbed, and
  KCES-pruned graphs.
  (b) Test accuracy under different pruning strategies on
  \textit{Metattack}-perturbed graphs.}
  \label{fig:mechanism}
\end{figure}

\vspace{0.5em}
Figure~\ref{fig:mechanism}(a) shows that \textit{Metattack} shifts KC scores toward larger values, creating a heavier high-KC tail. After high-KC pruning, the distribution moves closer to that of the clean graph, suggesting that harmful perturbations are enriched in the high-KC region under attack. This supports the link between KC scores and generalization degradation in Corollary~\ref{edge_theo}. Figure~\ref{fig:mechanism}(b) further shows that High-KC Pruning consistently improves or preserves accuracy, whereas Low-KC Pruning substantially degrades performance and Random Pruning lies in between. Overall, the pruning effectiveness follows \textbf{High-KC $>$ Random $>$ Low-KC}, indicating that high-KC pruning effectively mitigates harmful structural perturbations in attacked graphs.

\begin{table*}[t]
\caption{Defense performance (Accuracy $\pm$ Std) under structural attacks. Non-targeted and random attacks use a perturbation budget of 25\%, allowing modifications to up to 25\% of edges. \textit{Nettack} targets nodes with degree greater than 10 and perturbs their incident edges. Bold indicates the best performance, and “–” denotes inapplicable settings.}
\label{tab-1:robustness}
\centering
\small
\setlength{\tabcolsep}{3.2pt}
\renewcommand{\arraystretch}{1.12}

\rowcolors{2}{white}{tableZebra}

\resizebox{\textwidth}{!}{%
\begin{tabular}{llcccccccccc c}
\toprule
\rowcolor{white} \textbf{Dataset} & \textbf{Attack} &
\textbf{GCN} & \textbf{GAT} & \textbf{RGCN} & \textbf{Pro-G} & \textbf{G-SVD} &
\textbf{G-Jac} & \textbf{G-Guard} & \textbf{AirG} &
\textbf{NoiseG} & \textbf{GPR} & \cellcolor{myBlue} \textbf{KCES (Ours)} \\
\midrule

\cellcolor{white} & \textit{Clean} & $95.71 \pm 0.79$ & $95.40 \pm 0.41$ & $95.29 \pm 0.52$ & $95.60 \pm 0.46$ & $94.68 \pm 0.88$ & -- & -- & -- & $95.90 \pm 0.98$ & $94.74 \pm 0.40$ & \cellcolor{myBlue} $\mathbf{96.01} \pm \mathbf{0.18}$ \\
\cellcolor{white} & \textit{Random} & $81.29 \pm 0.61$ & $85.01 \pm 0.02$ & $82.23 \pm 0.76$ & $86.50 \pm 0.12$ & $88.34 \pm 0.83$ & -- & -- & -- & $83.82 \pm 0.88$ & $88.45 \pm 0.63$ & \cellcolor{myBlue} $\mathbf{89.46} \pm \mathbf{0.77}$ \\
\cellcolor{white} & \textit{Nettack} & $92.59 \pm 0.73$ & $85.79 \pm 0.23$ & $93.15 \pm 0.70$ & $95.20 \pm 0.59$ & $95.37 \pm 0.57$ & -- & -- & -- & $94.33 \pm 0.25$ & $92.96 \pm 0.06$ & \cellcolor{myBlue} $\mathbf{96.48} \pm \mathbf{0.82}$ \\
\cellcolor{white} \multirow{-2}{*}{\makecell[l]{\textit{Polblogs}\\(\textit{small})}} & \textit{DICE} & $71.47 \pm 0.98$ & $77.10 \pm 0.87$ & $71.16 \pm 0.36$ & $74.74 \pm 0.18$ & $76.48 \pm 0.73$ & -- & -- & -- & $73.78 \pm 0.35$ & $76.48 \pm 0.39$ & \cellcolor{myBlue} $\mathbf{93.04} \pm \mathbf{0.06}$ \\
\cellcolor{white} & \textit{MINMAX} & $70.14 \pm 0.54$ & $68.04 \pm 0.89$ & $82.10 \pm 0.42$ & $60.73 \pm 0.37$ & $62.47 \pm 0.72$ & -- & -- & -- & $62.99 \pm 0.70$ & $59.71 \pm 0.87$ & \cellcolor{myBlue} $\mathbf{94.17} \pm \mathbf{0.47}$ \\
\cellcolor{white} & \textit{Metattack} & $63.09 \pm 0.12$ & $63.60 \pm 0.20$ & $62.67 \pm 0.13$ & $65.23 \pm 0.24$ & $81.28 \pm 0.58$ & -- & -- & -- & $62.58 \pm 0.91$ & $61.61 \pm 0.31$ & \cellcolor{myBlue} $\mathbf{82.72} \pm \mathbf{0.32}$ \\
\midrule

\cellcolor{white} & \textit{Clean} & $83.65 \pm 0.57$ & $83.90 \pm 0.15$ & $82.89 \pm 0.99$ & $\mathbf{85.16} \pm \mathbf{0.77}$ & $78.16 \pm 0.07$ & $82.69 \pm 0.74$ & $78.87 \pm 0.77$ & $80.14 \pm 0.84$ & $82.39 \pm 0.42$ & $84.46 \pm 0.28$ & \cellcolor{myBlue} $84.04 \pm 0.42$ \\
\cellcolor{white} & \textit{Random} & $77.57 \pm 0.38$ & $79.23 \pm 0.75$ & $74.55 \pm 0.13$ & $79.58 \pm 0.32$ & $78.87 \pm 0.36$ & $77.36 \pm 0.93$ & $77.11 \pm 0.57$ & $65.83 \pm 0.18$ & $75.50 \pm 0.51$ & $79.03 \pm 0.14$ & \cellcolor{myBlue} $\mathbf{79.67} \pm \mathbf{0.69}$ \\
\cellcolor{white} & \textit{Nettack} & $57.83 \pm 0.78$ & $58.23 \pm 0.08$ & $59.04 \pm 0.51$ & $69.67 \pm 0.68$ & $74.70 \pm 1.21$ & $76.69 \pm 0.64$ & $62.65 \pm 0.03$ & $67.47 \pm 0.72$ & $60.42 \pm 0.04$ & $58.01 \pm 0.75$ & \cellcolor{myBlue} $\mathbf{82.24} \pm \mathbf{0.43}$ \\
\cellcolor{white} \multirow{-2}{*}{\makecell[l]{\textit{Cora}\\(\textit{small})}} & \textit{DICE} & $76.16 \pm 0.69$ & $77.06 \pm 0.31$ & $73.59 \pm 0.91$ & $75.95 \pm 0.40$ & $72.68 \pm 0.72$ & $77.11 \pm 0.22$ & $75.50 \pm 0.26$ & $72.24 \pm 0.73$ & $73.59 \pm 0.23$ & $77.76 \pm 0.19$ & \cellcolor{myBlue} $\mathbf{82.59} \pm \mathbf{0.09}$ \\
\cellcolor{white} & \textit{MINMAX} & $60.66 \pm 0.23$ & $61.26 \pm 0.73$ & $59.80 \pm 0.75$ & $64.73 \pm 0.23$ & $59.45 \pm 0.10$ & $72.23 \pm 0.69$ & $70.57 \pm 0.95$ & $64.73 \pm 0.86$ & $61.82 \pm 0.93$ & $62.42 \pm 0.55$ & \cellcolor{myBlue} $\mathbf{78.42} \pm \mathbf{0.47}$ \\
\cellcolor{white} & \textit{Metattack} & $53.12 \pm 0.83$ & $58.35 \pm 0.22$ & $51.35 \pm 0.72$ & $63.37 \pm 0.16$ & $61.92 \pm 0.28$ & $75.28 \pm 0.19$ & $70.77 \pm 0.97$ & $63.83 \pm 0.92$ & $55.18 \pm 0.79$ & $56.24 \pm 0.53$ & \cellcolor{myBlue} $\mathbf{82.99} \pm \mathbf{0.08}$ \\
\midrule

\cellcolor{white} & \textit{Clean} & $72.51 \pm 0.61$ & $72.69 \pm 0.55$ & $71.97 \pm 0.60$ & $71.74 \pm 0.59$ & $69.60 \pm 0.56$ & $72.98 \pm 0.06$ & $71.03 \pm 0.19$ & $72.23 \pm 0.39$ & $71.33 \pm 0.58$ & $72.16 \pm 0.46$ & \cellcolor{myBlue} $\mathbf{73.02} \pm \mathbf{0.93}$ \\
\cellcolor{white} & \textit{Random} & $70.38 \pm 0.83$ & $69.31 \pm 0.95$ & $67.06 \pm 0.26$ & $72.36 \pm 0.19$ & $67.59 \pm 0.15$ & $71.21 \pm 0.03$ & $72.73 \pm 0.60$ & $65.28 \pm 0.70$ & $69.43 \pm 0.97$ & $68.96 \pm 0.01$ & \cellcolor{myBlue} $\mathbf{72.81} \pm \mathbf{0.62}$ \\
\cellcolor{white} & \textit{Nettack} & $52.38 \pm 0.94$ & $59.19 \pm 0.46$ & $49.21 \pm 0.97$ & $72.23 \pm 1.04$ & $74.60 \pm 0.51$ & $72.14 \pm 0.69$ & $72.95 \pm 0.58$ & $\mathbf{77.78} \pm \mathbf{0.77}$ & $63.49 \pm 0.71$ & $50.03 \pm 0.62$ & \cellcolor{myBlue} $76.14 \pm 0.60$ \\
\cellcolor{white} \multirow{-2}{*}{\makecell[l]{\textit{Citeseer}\\(\textit{small})}} & \textit{DICE} & $67.71 \pm 0.72$ & $66.60 \pm 0.43$ & $66.17 \pm 0.85$ & $72.15 \pm 0.61$ & $67.35 \pm 0.03$ & $71.14 \pm 0.19$ & $69.01 \pm 0.34$ & $67.20 \pm 0.48$ & $67.24 \pm 0.93$ & $67.54 \pm 0.89$ & \cellcolor{myBlue} $\mathbf{72.45} \pm \mathbf{0.84}$ \\
\cellcolor{white} & \textit{MINMAX} & $66.29 \pm 0.45$ & $67.54 \pm 0.43$ & $61.02 \pm 0.44$ & $69.90 \pm 0.80$ & $64.57 \pm 0.86$ & $71.20 \pm 0.02$ & $68.60 \pm 0.86$ & $66.28 \pm 0.32$ & $70.08 \pm 0.22$ & $68.48 \pm 0.49$ & \cellcolor{myBlue} $\mathbf{72.80} \pm \mathbf{0.36}$ \\
\cellcolor{white} & \textit{Metattack} & $57.64 \pm 0.59$ & $61.20 \pm 0.66$ & $56.81 \pm 0.75$ & $66.33 \pm 0.46$ & $66.29 \pm 0.39$ & $70.14 \pm 0.53$ & $64.75 \pm 0.60$ & $65.23 \pm 0.34$ & $59.94 \pm 0.94$ & $59.48 \pm 0.43$ & \cellcolor{myBlue} $\mathbf{71.86} \pm \mathbf{0.87}$ \\
\midrule

\cellcolor{white} & \textit{Clean} & $85.72 \pm 0.05$ & $84.89 \pm 0.84$ & $84.72 \pm 0.06$ & $85.19 \pm 0.43$ & $84.53 \pm 0.86$ & $\mathbf{86.19} \pm \mathbf{0.97}$ & $84.49 \pm 0.79$ & $84.85 \pm 0.40$ & $85.11 \pm 0.45$ & $85.07 \pm 0.88$ & \cellcolor{myBlue} $86.17 \pm 0.52$ \\
\cellcolor{white} & \textit{Random} & $84.11 \pm 0.28$ & $81.02 \pm 0.72$ & $83.75 \pm 0.10$ & $84.28 \pm 0.41$ & $82.61 \pm 0.63$ & $84.27 \pm 0.64$ & $83.87 \pm 0.03$ & $83.09 \pm 0.69$ & $83.50 \pm 0.52$ & $83.02 \pm 0.59$ & \cellcolor{myBlue} $\mathbf{85.83} \pm \mathbf{0.82}$ \\
\cellcolor{white} & \textit{Nettack} & $66.67 \pm 0.36$ & $76.73 \pm 0.88$ & $72.58 \pm 0.09$ & $72.60 \pm 0.14$ & $80.10 \pm 0.45$ & $85.48 \pm 0.12$ & $83.33 \pm 0.65$ & $85.48 \pm 0.96$ & $65.44 \pm 0.64$ & $67.37 \pm 0.30$ & \cellcolor{myBlue} $\mathbf{86.24} \pm \mathbf{0.09}$ \\
\cellcolor{white} \multirow{-2}{*}{\makecell[l]{\textit{Pubmed}\\(\textit{medium})}} & \textit{DICE} & $81.68 \pm 0.46$ & $76.93 \pm 0.19$ & $81.44 \pm 0.56$ & $80.73 \pm 0.30$ & $80.39 \pm 0.12$ & $82.93 \pm 0.70$ & $82.28 \pm 0.70$ & $83.09 \pm 0.27$ & $81.42 \pm 0.08$ & $78.86 \pm 0.16$ & \cellcolor{myBlue} $\mathbf{85.74} \pm \mathbf{0.39}$ \\
\cellcolor{white} & \textit{MINMAX} & $55.67 \pm 0.46$ & $60.01 \pm 0.97$ & $54.64 \pm 0.24$ & $69.29 \pm 0.37$ & $80.50 \pm 0.06$ & $84.51 \pm 0.83$ & $81.69 \pm 0.57$ & $84.02 \pm 0.13$ & $57.51 \pm 0.81$ & $56.32 \pm 0.68$ & \cellcolor{myBlue} $\mathbf{85.64} \pm \mathbf{0.97}$ \\
\cellcolor{white} & \textit{Metattack} & $46.08 \pm 0.39$ & $49.72 \pm 0.37$ & $45.99 \pm 0.92$ & $72.08 \pm 0.51$ & $82.75 \pm 0.25$ & $84.22 \pm 0.39$ & $83.37 \pm 0.89$ & $84.83 \pm 0.20$ & $47.17 \pm 0.29$ & $48.32 \pm 0.66$ & \cellcolor{myBlue} $\mathbf{86.45} \pm \mathbf{0.45}$ \\
\midrule

\cellcolor{white} & \textit{Clean} & $56.24 \pm 0.81$ & $47.83 \pm 0.25$ & $39.62 \pm 0.73$ & $54.68 \pm 0.68$ & $60.46 \pm 0.55$ & $74.04 \pm 0.38$ & $74.31 \pm 0.74$ & $75.56 \pm 0.10$ & $60.91 \pm 0.74$ & $62.62 \pm 0.15$ & \cellcolor{myBlue} $\mathbf{76.25} \pm \mathbf{0.93}$ \\
\cellcolor{white} & \textit{Random} & $62.82 \pm 0.62$ & $60.80 \pm 0.42$ & $62.44 \pm 0.16$ & $65.43 \pm 0.90$ & $76.54 \pm 0.41$ & $74.83 \pm 0.96$ & $74.85 \pm 0.55$ & $76.02 \pm 0.44$ & $69.13 \pm 0.21$ & $50.03 \pm 0.50$ & \cellcolor{myBlue} $\mathbf{76.74} \pm \mathbf{0.65}$ \\
\cellcolor{white} & \textit{Nettack} & $38.82 \pm 0.55$ & $43.12 \pm 0.39$ & $59.87 \pm 0.36$ & $70.31 \pm 0.74$ & $58.70 \pm 0.45$ & $72.58 \pm 0.48$ & $74.51 \pm 0.33$ & $\mathbf{75.63} \pm \mathbf{0.03}$ & $49.67 \pm 0.80$ & $39.82 \pm 0.95$ & \cellcolor{myBlue} $73.47 \pm 0.97$ \\
\cellcolor{white} \multirow{-2}{*}{\makecell[l]{\textit{Flickr}\\(\textit{large})}} & \textit{DICE} & $51.71 \pm 0.32$ & $48.11 \pm 0.11$ & $47.34 \pm 0.39$ & $71.23 \pm 0.17$ & $73.38 \pm 0.93$ & $73.35 \pm 0.02$ & $73.59 \pm 0.24$ & $73.08 \pm 0.83$ & $51.52 \pm 0.57$ & $52.09 \pm 0.90$ & \cellcolor{myBlue} $\mathbf{73.69} \pm \mathbf{0.67}$ \\
\cellcolor{white} & \textit{MINMAX} & $14.57 \pm 0.97$ & $11.71 \pm 0.26$ & $27.13 \pm 0.02$ & $19.68 \pm 0.58$ & $39.17 \pm 0.78$ & $74.82 \pm 0.33$ & $75.04 \pm 0.08$ & $75.45 \pm 0.65$ & $28.27 \pm 0.12$ & $9.80 \pm 0.67$ & \cellcolor{myBlue} $\mathbf{76.86} \pm \mathbf{0.90}$ \\
\cellcolor{white} & \textit{Metattack} & $36.93 \pm 0.86$ & $37.29 \pm 0.06$ & $31.72 \pm 0.21$ & $65.05 \pm 0.82$ & $59.08 \pm 0.36$ & $74.85 \pm 0.95$ & $75.05 \pm 0.75$ & $75.90 \pm 0.05$ & $49.67 \pm 0.38$ & $30.92 \pm 0.07$ & \cellcolor{myBlue} $\mathbf{76.63} \pm \mathbf{0.82}$ \\
\bottomrule
\end{tabular}}
\end{table*}
\vspace{-0.5em}
\subsection{Defense against structural attacks}
\label{defense-performance}

We evaluate KCES against various defenses following the protocol in Section~\ref{setup}. Unless otherwise specified, KCES is applied to a GCN backbone, with the pruning ratio $\alpha$ tuned via grid search over $[0.1, 0.9]$ on the validation set, and the best-performing configuration is reported. Detailed attack configurations are summarized in Table~\ref{tab-1:robustness}. All results are reported as percentages, and “–” indicates that a method is not applicable under the corresponding attack setting.

Table~\ref{tab-1:robustness} shows that KCES consistently outperforms all baselines, achieving state-of-the-art performance under most adversarial settings. Notably, under structural attacks, KCES often restores performance close to—or even exceeding—that on the clean graph, indicating that its edge sanitization effectively counteracts adversarial perturbations without compromising predictive accuracy. A particularly striking result appears on \textit{Flickr}, where KCES attains higher accuracy than on the original clean topology. This observation suggests that large-scale real-world graphs may contain noisy or redundant edges, consistent with prior findings~\cite{dai2022towards, dong2023towards}. By identifying and pruning such detrimental connections, KCES effectively serves as a structural regularizer, simultaneously improving adversarial robustness and clean generalization.

\subsection{Scalability and Efficiency Analysis}
\label{sec:scale_efficiency}

While Section~\ref{sec:kces-algorithm} provides a theoretical complexity analysis, it is crucial to verify the practical scalability of KCES on massive-scale graphs. In this section, we conduct a rigorous evaluation to demonstrate that KCES is computationally efficient and not limited to small benchmarks.
For the scalability test on the massive \textit{Ogbn-Arxiv} dataset (over 1M edges), we employ \textit{PRBCD}~\cite{geisler2021robustness}, a scalable gradient-based attack, with a perturbation budget ($ptb$) of 0.10, as traditional methods (e.g., \textit{Metattack}) are computationally infeasible at this scale.

To ensure a comprehensive comparison, we evaluate KCES from two aspects:
\textbf{(1) Robustness at Scale:} We report defense performance on
\textit{Ogbn-Arxiv} under \textit{PRBCD} in Table~\ref{tab:ogbn_arxiv_horizontal_clean}. \textbf{(2) Computational Overhead:} We profile time and memory usage on the medium-scale \textit{Pubmed} dataset under \textit{Metattack} ($ptb=0.10$) (Table~\ref{tab:runtime_space}). For a fair comparison across methods that operate at different stages, we report the total wall-clock time as the sum of method-specific graph preprocessing time and 200 training epochs. Thus, for preprocessing-based methods such as \textit{G-SVD}, \textit{G-Jac}, \textit{GPR}, and KCES, the reported time includes their preprocessing overhead before training, whereas methods without an explicit preprocessing stage have zero additional preprocessing cost. This unified measurement reflects the practical cost of using each defense pipeline. We use \textit{Pubmed} for this profiling because it allows us to include heavy optimization-based baselines such as \textit{ProGNN}, which may encounter Out-Of-Memory (OOM) failures on larger graphs. For completeness, the standalone optimized KC-score preprocessing time of KCES across different graph scales is reported in the supplementary material (Section~4.5).

\begin{table*}[htbp]
\centering
\caption{Performance (\%) on \textit{Ogbn-Arxiv} under clean setting and \textit{PRBCD} attack.}
\label{tab:ogbn_arxiv_horizontal_clean}
\small
\setlength{\tabcolsep}{2pt} 
\renewcommand{\arraystretch}{1.3} 

\rowcolors{2}{tableZebra}{white}

\begin{tabularx}{\textwidth}{l *{10}{>{\centering\arraybackslash}X} >{\centering\arraybackslash}p{4em}}
\toprule
\rowcolor{white}
\textbf{Metric} 
& \textbf{GCN} & \textbf{GAT} & \textbf{RGCN} & \textbf{Pro-G} 
& \textbf{G-SVD} & \textbf{G-Jac} & \textbf{G-Guard} 
& \textbf{AirG} & \textbf{NoiseG} & \textbf{GPR} & \cellcolor{myBlue} \textbf{KCES} \\
\midrule
\textit{Clean} 
& 67.51 & 66.54 & 68.05 & OOM & 58.62 & 67.42 & 66.45 & 66.96 & 68.92 & 68.16 & \cellcolor{myBlue} \textbf{69.32} \\

\textit{PRBCD} 
& 40.66 & 42.53 & 47.27 & OOM & 53.94 & 38.76 & 45.24 & 40.96 & 37.24 & 42.13 & \cellcolor{myBlue} \textbf{58.62} \\
\bottomrule
\end{tabularx}
\end{table*}
\begin{table*}[htbp]
\centering
\caption{Time and space complexity comparison on \textbf{\textit{Pubmed}} under \textit{Metattack}. Total runtime is measured over 200 training epochs.}
\label{tab:runtime_space}
\small
\setlength{\tabcolsep}{2pt} 
\renewcommand{\arraystretch}{1.3} 

\rowcolors{2}{tableZebra}{white}

\begin{tabularx}{\textwidth}{l *{10}{>{\centering\arraybackslash}X} >{\centering\arraybackslash}p{4em}}
\toprule
\rowcolor{white} 
\textbf{Metric} 
& \textbf{GCN} & \textbf{GAT} & \textbf{RGCN} & \textbf{Pro-G} 
& \textbf{G-SVD} & \textbf{G-Jac} & \textbf{G-Gd} 
& \textbf{AirG} & \textbf{NoiseG} & \textbf{GPR} & \cellcolor{myBlue} \textbf{KCES} \\
\midrule
\textit{Space (MB)}
& 1,484 & 5,645 & 8,950 & 17,518 & 6,266 & 1,490 & 9,689 & 3,304 & 1,728 & 3,456 & \cellcolor{myBlue} 1,539 \\

\textit{Time (s)}
& 2.30 & 6.51 & 9.56 & 7,340.23 & 6.55 & 3.46 & 415.23 & 2.44 & 2.68 & 24.56 & \cellcolor{myBlue} 6.40 \\
\bottomrule
\end{tabularx}
\end{table*}

The results provide strong evidence of KCES’s scalability and efficiency.  On \textit{Ogbn-Arxiv} (Table~\ref{tab:ogbn_arxiv_horizontal_clean}), optimization-based defenses such as \textit{ProGNN} fail to execute due to memory constraints (OOM), whereas KCES scales successfully and achieves state-of-the-art robustness (58.62\%), outperforming scalable heuristics such as \textit{GNN-SVD} and \textit{GNN-Jaccard}. As shown in Table~\ref{tab:runtime_space}, KCES incurs only minimal computational overhead and remains memory-efficient, in sharp contrast to iterative optimization methods that require orders of magnitude more time and memory (e.g., \textit{ProGNN} takes over $7{,}000$s). These results demonstrate that KCES enables large-scale structural defense without the heavy computational burden typical of optimization-based approaches.

\subsection{Sensitivity Analysis of Pseudo-Label Generation}
\label{subsec:sensitivity}
We study the sensitivity of KCES to pseudo-label generation by varying both the number of clusters and the clustering algorithm. Experiments are conducted on \textit{Cora} and \textit{Citeseer} under \textit{Metattack} with a perturbation budget of 0.25. The cluster number $K$ ranges from 2 to 30. In addition, fixing $K=6$, we compare three clustering methods: \textbf{Spectral Clustering}~\cite{luxburg2007tutorial}, \textbf{Gaussian Mixture Models (GMM)}~\cite{bishop2006pattern}, and \textbf{K-Means}~\cite{macqueen1967some}.

\begin{table}[htbp]
\centering
\caption{Impact of the number of clusters $K$ on KCES.}
\label{tab:k_sensitivity}
\small
\setlength{\tabcolsep}{4pt} 
\renewcommand{\arraystretch}{1.3} 

\rowcolors{2}{white}{tableZebra}

\begin{tabularx}{\columnwidth}{l *{6}{>{\centering\arraybackslash}X}}
\toprule
\rowcolor{white} \textbf{Dataset} & \textbf{2} & \textbf{5} & \textbf{10} & \textbf{15} & \textbf{20} & \textbf{30} \\
\midrule
\textit{Cora}     & 80.31 & 81.25 & 80.92 & 80.25 & 80.51 & 79.86 \\
\textit{Citeseer} & 71.50 & 71.20 & 71.03 & 70.84 & 71.29 & 70.97 \\
\bottomrule
\end{tabularx}
\end{table}

\begin{table}[htbp]
\centering
\caption{Impact of different clustering algorithms on KCES.}
\label{tab:cluster_algo}
\small
\setlength{\tabcolsep}{4pt} 
\renewcommand{\arraystretch}{1.3} 

\rowcolors{2}{white}{tableZebra}

\begin{tabularx}{\columnwidth}{l *{4}{>{\centering\arraybackslash}X}}
\toprule
\rowcolor{white} \textbf{Dataset} & \textbf{Attack} & \textbf{Spectral} & \textbf{GMM} & \textbf{K-Means} \\
\midrule
\textit{Cora}    & 52.30 & 80.03 & 80.51 & 80.11 \\
\textit{Citeseer} & 56.59 & 71.03 & 72.20 & 72.79 \\
\bottomrule
\end{tabularx}
\end{table}

Results in Table~\ref{tab:k_sensitivity} show that KCES is largely insensitive to the cluster number $K$, exhibiting stable performance across a wide range of values. This suggests that KCES relies on detecting \textbf{structural inconsistencies} rather than recovering exact semantics. Theoretically, adversarial edges typically bridge communities and thus span cluster boundaries regardless of granularity (whether $K=2$ or $K=30$). Consequently, \textbf{while absolute KC scores may shift with $\mathbf{\hat{y}}$, the relative ranking of structurally detrimental edges remains stable}. This stability justifies using pseudo-labels $\mathbf{\hat{y}}$ as a proxy for $\mathbf{y}$. Furthermore, the robustness across algorithms in Table~\ref{tab:cluster_algo} confirms that KCES exploits fundamental structural discrepancy signals—specifically the violation of local smoothness—rather than artifacts of specific clustering configurations.

\subsection{Plug-and-Play Compatibility}
\label{subsec:plug-and-play}

KCES can also be used as a preprocessing module for existing defenses. To
evaluate this plug-and-play property, we apply KCES before representative
methods, including GAT, RGCN, ProGNN, GNN-SVD, GCN-Jaccard, and GNNGuard.
We denote the KCES-enhanced variant by ``(K)''.

\begin{figure*}[t]
  \centering
  \includegraphics[width=1.0\linewidth]{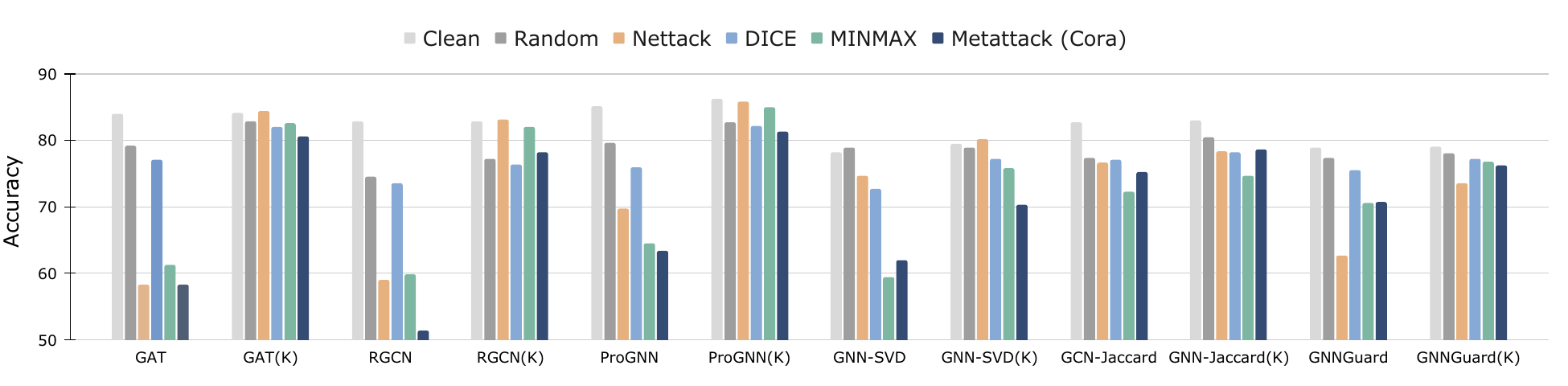}
  \caption{Plug-and-play compatibility of KCES on Cora. ``(K)'' denotes applying
  KCES before the corresponding defense. KCES generally improves robustness
  under structural attacks while preserving clean accuracy.}
  \label{fig-5:plug-and-play}
\end{figure*}

Figure~\ref{fig-5:plug-and-play} shows that KCES consistently improves or maintains the performance of existing defenses across clean and attacked graphs. The improvements are especially clear under stronger structural attacks such as \textit{Nettack}, \textit{MINMAX}, and \textit{Metattack}, indicating that KCES effectively removes harmful edges before downstream defense models are trained. These results confirm that KCES is complementary to existing robust GNN methods rather than merely a standalone defense. Full results are provided in the supplementary material.

\section{Related Works}
\label{sec:related_works}

\paragraph{Robustness to Structural Attacks in GNNs}
Adversarial attacks aim to degrade model performance through subtle input perturbations~\cite{szegedy2013intriguing, madry2017towards, papernot2016limitations, kurakin2018adversarial, moosavi2016deepfool}. On graphs, structural attacks directly modify the topology~\cite{zugner2018adversarial, xu2019topology, waniek2018hiding, alom2025gottack}, posing a severe threat to message-passing GNNs. Existing defenses mainly fall into three categories. \textit{Graph purification} methods denoise the input structure, including GNN-Jaccard~\cite{wu2019adversarial}, which filters edges by feature similarity, GNN-SVD~\cite{entezari2020all}, which applies low-rank approximation, and GPR-GAE~\cite{lee2025self}, which uses a self-supervised graph auto-encoder. \textit{Robust architecture} methods improve model resilience through structural or propagation design, such as ProGNN~\cite{jin2020graph}, AirGNN~\cite{liu2021graph}, and NoiseGNN~\cite{ennadir2024simple}. \textit{Graph adversarial training}, exemplified by RGCN~\cite{zhu2019robust}, trains models on perturbed graphs. However, many defenses rely on heuristic assumptions, generalize poorly to unseen attacks, or require costly retraining and optimization. These limitations motivate a principled, effective, and training-free defense.

\paragraph{Gram Matrix Applications}
Gram matrices are widely used to analyze neural networks from both model and optimization perspectives. They help study how architectures learn target functions~\cite{rahimi2007random}, characterize invariance properties in MLPs~\cite{tsuchida2018invariance}, and explain training dynamics in over-parameterized or two-layer networks~\cite{allen2019convergence, arora2019fine}. Recent training-free data valuation methods further use Gram matrices to measure the influence of individual data points in Euclidean domains~\cite{nohyun2022data}. However, such ideas remain underexplored for GNN robustness. Inspired by these advances~\cite{arora2019fine, nohyun2022data}, we extend Gram-matrix-based reasoning to graph-structured data by deriving a kernel representation from graph aggregation, leading to KCES, a training-free and model-agnostic defense against structural perturbations.

\paragraph{Kernel-view GNN Theory}
GNN generalization has been studied through classical capacity measures such as Rademacher complexity~\cite{garg2020generalization}, as well as kernel-based perspectives. The Graph Neural Tangent Kernel (GNTK) shows that infinite-width GNNs converge to architecture-specific kernels, with subsequent work further developing this view~\cite{du2019graph, cosmo2024graph, krishnagopal2023graph, zhou2023explainability, tang2022graphqntk}. These approaches are largely model-centric: their kernels describe the behavior of specific GNN architectures rather than providing a direct data-centric measure of graph complexity. In contrast, we introduce Graph Kernel Complexity (GKC), a computable metric obtained by embedding graph aggregation into the kernel definition. GKC is computed from the graph-feature pair and pseudo-labels, is independent of specific GNN parameters, and appears in the test error bound in Theorem~\ref{thm:test-error}. This data-centric view enables KCES to improve structural robustness through training-free edge sanitization. A detailed comparison with prior kernel methods is provided in the supplementary material (Section~1).

\section{Conclusion}
\label{sec:conclusion}

In this work, we presented Kernel-Complexity Edge Sanitization (KCES), a training-free and model-agnostic framework that improves the structural robustness of GNNs by connecting graph-induced kernel complexity with generalization and performing targeted edge sanitization via generalization-guided KC scores. As a lightweight preprocessing module requiring no retraining, KCES is computationally efficient, scalable to large graphs, compatible with diverse GNN architectures, and empirically achieves consistent robustness gains while largely preserving clean accuracy. More broadly, our results suggest that graph robustness can benefit from complexity-aware preprocessing, where structural modifications are guided by their influence on graph-induced kernel geometry rather than only by local similarity or reconstruction heuristics. This perspective opens a promising direction for designing robust graph learning systems that combine theoretical generalization signals with practical, scalable graph sanitization.

\paragraph{Limitations.} KCES is designed for structural perturbations and is therefore not a direct defense against node- or feature-level attacks; consistent with this scope, feature-perturbation experiments in the supplementary material (Section~4.4) show only limited gains. Since KCES relies on pseudo-labels induced by aggregated node features, its benefits may be less pronounced on heterophilous graphs where local aggregation is less aligned with class-homophilic smoothness (Section~4.9 of the supplementary material), and as a pruning-based method, it may discard useful information under highly structured perturbations. Future work could extend KCES toward feature-aware or edge-reweighting variants to better handle non-structural attacks and highly structured perturbation patterns.

\section*{Ethics and Privacy Statement}

This work uses only publicly available benchmark datasets and models and does not involve human subjects, private information, or sensitive user data. We do not identify direct privacy or ethical concerns associated with the experiments.

\appendix
\section{Experimental Setup}
\label{sec:exp-setup}

We implement attack and defense baselines using the \textbf{DeepRobust}
library~\cite{jin2020graph}. Unless otherwise specified, all models are trained
for 200 epochs with ReLU activation and Adam optimizer, using a learning rate
of 0.01 and weight decay of $1\times10^{-5}$. Dataset-specific hyperparameters
are summarized in Table~\ref{tab:hyper}. All experiments are conducted on a
dedicated server with eight NVIDIA RTX A6000 GPUs, each with 48GB memory.

\begin{table}[htbp]
\centering
\caption{Hyperparameters used in our experiments.}
\label{tab:hyper}
\small
\setlength{\tabcolsep}{4pt} 
\renewcommand{\arraystretch}{1.3} 
\rowcolors{2}{white}{tableZebra}

\begin{tabularx}{\columnwidth}{l *{5}{>{\centering\arraybackslash}X}}
\toprule
\rowcolor{white} & \textbf{\textit{Polblogs}} & \textbf{\textit{Cora}} & \textbf{\textit{Citeseer}} & \textbf{\textit{Pubmed}} & \textbf{\textit{Flickr}} \\
\midrule
\textbf{\# Layers}         & 2 & 2 & 2 & 2 & 2 \\
\textbf{Hidden Dim.}       & [16, 2] & [16, 7] & [16, 6] & [32, 3] & [16, 9] \\
\textbf{Dropout}           & 0.05 & 0.05 & 0.05 & 0.05 & 0.05 \\
\bottomrule
\end{tabularx}
\end{table}


\section{Theoretical Analysis of the Graph Kernel Model}
\label{sec:theory-gkm}

This section provides the formal justification for the graph kernel
complexity bound used in the main paper. Our analysis follows the
two-layer ReLU kernel generalization framework of~\cite{arora2019fine},
but instantiates it on graph-aggregated node features. We first state the
graph-kernel setup and assumptions, then recall the reference kernel
results, and finally derive our training, test, and edge-specific bounds.

\subsection{Graph Kernel Setup}
\label{subsec:theory-setup}

Consider an undirected graph
\(G=(X,\tilde A,\tilde D,\mathbf y)\) with \(N\) nodes, node feature
matrix \(X\in\mathbb R^{N\times F}\), adjacency matrix
\(\tilde A\in\{0,1\}^{N\times N}\) including self-loops, and degree
matrix \(\tilde D_{ii}=\sum_j\tilde A_{ij}\). We define the normalized
graph aggregation operator and the aggregated node features as
\begin{equation}
    T = \tilde D^{-\frac12}\tilde A\tilde D^{-\frac12},
    \qquad
    \tilde X = TX .
\end{equation}

We analyze the following two-layer graph kernel model:
\begin{equation}
\label{eq:gnn-kernel-model}
f_{\mathrm{GNN}}(X_i,\tilde A,\tilde D)
=
\frac{1}{\sqrt m}
\sum_{r=1}^m
a_r \sigma\!\left(W_r^\top \tilde X_i\right),
\end{equation}
where \(m\) is the number of hidden units, \(W_r\in\mathbb R^F\) is the
first-layer weight of the \(r\)-th neuron, \(a_r\in\{-1,1\}\) is the
fixed second-layer coefficient, and \(\sigma(\cdot)\) is the ReLU
activation.

The empirical training loss is
\begin{equation}
\label{eq:graph-train-loss}
L(W)
=
\frac12
\sum_{i=1}^{N}
\left(
y_i - f_{\mathrm{GNN}}(X_i,\tilde A,\tilde D)
\right)^2,
\end{equation}
and the expected test loss is
\begin{equation}
\label{eq:graph-test-loss}
L_{\mathcal D_G}(W)
=
\mathbb E_{G\sim\mathcal D_G}[L(W)] .
\end{equation}

The infinite-width graph kernel Gram matrix is defined as
\begin{equation}
\label{eq:graph-kernel-gram}
H^\infty_{ij}
=
\mathbb E_{w\sim\mathcal N(0,I)}
\left[
\tilde X_i^\top \tilde X_j
\,
\mathbb I
\{w^\top \tilde X_i\ge 0,\,
  w^\top \tilde X_j\ge 0\}
\right].
\end{equation}
When \(\|\tilde X_i\|_2=1\), this admits the closed form
\begin{equation}
\label{eq:graph-kernel-closed-form}
H^\infty_{ij}
=
\frac{
\tilde X_i^\top \tilde X_j
\left(
\pi-\arccos(\tilde X_i^\top \tilde X_j)
\right)
}{2\pi}.
\end{equation}

The Graph Kernel Complexity (GKC) is defined as
\begin{equation}
\label{eq:gkc-appendix}
\mathrm{GKC}(H^\infty,\mathbf y)
=
\frac{2\mathbf y^\top (H^\infty)^{-1}\mathbf y}{N}.
\end{equation}

\subsection{Assumptions}
\label{subsec:theory-assumptions}

\begin{assumption}
\label{assum-e-2}
The following conditions hold throughout the analysis.
\begin{enumerate}[leftmargin=1.5em]
    \item \textbf{Normalized aggregated features.}
    For every node \(i\in[N]\), \(\|\tilde X_i\|_2=1\).

    \item \textbf{Bounded labels.}
    For every node \(i\in[N]\), \(|y_i|\le 1\).

    \item \textbf{Non-degenerate graph kernel.}
    With probability at least \(1-\delta\), the graph kernel Gram matrix
    satisfies
    \begin{equation}
        \lambda_{\min}(H^\infty)\ge \lambda_0>0 .
    \end{equation}

    \item \textbf{Edge-robust non-degeneracy.}
    For every candidate edge \(e_{ij}\) considered by KCES, let
    \(G_{-(i,j)}\) be the graph obtained by removing \(e_{ij}\), and let
    \(H^\infty_{-(i,j)}\) be the corresponding graph kernel Gram matrix.
    With probability at least \(1-\delta\),
    \begin{equation}
        \lambda_{\min}(H^\infty_{-(i,j)})\ge \lambda_0>0 .
    \end{equation}

    \item \textbf{Initialization.}
    The first-layer weights are initialized as
    \(W_r(0)\sim\mathcal N(0,\kappa^2 I)\), where \(0<\kappa\le 1\).
    The second-layer coefficients are initialized as
    \(a_r\sim\mathrm{Unif}(\{-1,1\})\) and kept fixed.

    \item \textbf{Gradient descent.}
    Only the first-layer weights are optimized, using
    \begin{equation}
        W_{t+1}=W_t-\eta\nabla_W L(W_t).
    \end{equation}
\end{enumerate}
\end{assumption}

\paragraph{Remark on non-degeneracy.}
The non-degeneracy conditions above are the graph-kernel counterparts of
the positive-definiteness assumption required by standard ReLU kernel
generalization bounds. They ensure that both the original graph and the
edge-deleted graphs considered by KCES induce well-conditioned kernel
matrices. This is the only graph-specific condition needed to instantiate
the reference kernel results on graph-aggregated features.

\subsection{Reference Kernel Results}
\label{subsec:reference-results}

We recall the two kernel results from~\cite{arora2019fine} that our
analysis relies on. Consider a two-layer ReLU network trained on inputs
\(\{z_i\}_{i=1}^n\), labels \(\mathbf s\), and infinite-width kernel
matrix \(H^{\mathrm{ref}}\). If the inputs are normalized, labels are
bounded, and
\(\lambda_{\min}(H^{\mathrm{ref}})\ge \lambda_0^{\mathrm{ref}}>0\),
then gradient descent satisfies the following training and test error
guarantees.

\begin{lemma}[Reference training dynamics]
\label{lem:ref-training}
Under the reference assumptions of~\cite{arora2019fine}, let
\[
\kappa
=
O\!\left(\frac{\epsilon\delta}{\sqrt n}\right),
\qquad
m
=
\Omega\!\left(
\frac{n^7}{(\lambda_0^{\mathrm{ref}})^4\kappa^2\delta^4\epsilon^2}
\right),
\qquad
\eta
=
O\!\left(\frac{\lambda_0^{\mathrm{ref}}}{n^2}\right).
\]
Then, with probability at least \(1-\delta\), for all \(t\ge 0\),
\begin{equation}
\label{eq:ref-training}
L^{\mathrm{ref}}(W_t)
=
\sqrt{
\sum_{\ell=1}^{n}
(1-\eta\lambda_\ell)^{2t}
(\mathbf v_\ell^\top \mathbf s)^2
}
\pm \epsilon,
\end{equation}
where \(\lambda_\ell\) and \(\mathbf v_\ell\) are the eigenvalues and
eigenvectors of \(H^{\mathrm{ref}}\).
\end{lemma}

\begin{lemma}[Reference test error bound]
\label{lem:ref-test}
Under the reference assumptions of~\cite{arora2019fine}, let
\[
\kappa
=
O\!\left(\frac{\lambda_0^{\mathrm{ref}}\delta}{n}\right),
\qquad
m
\ge
\kappa^{-2}\cdot
\mathrm{poly}\!\left(n,(\lambda_0^{\mathrm{ref}})^{-1},\delta^{-1}\right).
\]
For
\[
t
\ge
\Omega\!\left(
\frac{1}{\eta\lambda_0^{\mathrm{ref}}}
\log\frac{n}{\delta}
\right),
\]
the trained model satisfies, with probability at least \(1-\delta\),
\begin{equation}
\label{eq:ref-test}
L_{\mathcal D}^{\mathrm{ref}}(W_t)
\le
\sqrt{
\frac{2\mathbf s^\top (H^{\mathrm{ref}})^{-1}\mathbf s}{n}
}
+
O\!\left(
\sqrt{
\frac{\log\frac{n}{\lambda_0^{\mathrm{ref}}\delta}}{n}
}
\right).
\end{equation}
\end{lemma}

\subsection{Instantiation on Graph-Aggregated Features}
\label{subsec:graph-instantiation}

The graph kernel model in Eq.~\eqref{eq:gnn-kernel-model} is a two-layer
ReLU model applied to the graph-aggregated inputs
\(\{\tilde X_i\}_{i=1}^N\). Therefore, the reference kernel results can be
instantiated through the correspondence
\begin{equation}
    n \leftarrow N,
    \qquad
    z_i \leftarrow \tilde X_i,
    \qquad
    \mathbf s \leftarrow \mathbf y,
    \qquad
    H^{\mathrm{ref}} \leftarrow H^\infty,
    \qquad
    \lambda_0^{\mathrm{ref}} \leftarrow \lambda_0 .
\end{equation}

Under Assumption~\ref{assum-e-2}, the normalized-input, bounded-label,
positive-definiteness, initialization, fixed-second-layer, and gradient
descent conditions required by the reference theory are all satisfied.
Thus, Lemmas~\ref{lem:ref-training} and~\ref{lem:ref-test} directly yield
the graph-specific training and test error bounds below.

\subsection{Proof of the Training Error Bound}
\label{subsec:proof-training}

\begin{theorem}[Training error dynamics]
\label{train_theo_fo}
Under Assumption~\ref{assum-e-2}, let
\[
\kappa
=
O\!\left(\frac{\epsilon\delta}{\sqrt N}\right),
\qquad
m
=
\Omega\!\left(
\frac{N^7}{\lambda_0^4\kappa^2\delta^4\epsilon^2}
\right),
\qquad
\eta
=
O\!\left(\frac{\lambda_0}{N^2}\right).
\]
Then, with probability at least \(1-\delta\) over initialization, for all
\(t\ge 0\),
\begin{equation}
\label{eq:graph-training-bound}
L(W_t)
=
\sqrt{
\sum_{\ell=1}^{N}
(1-\eta\lambda_\ell)^{2t}
(\mathbf v_\ell^\top \mathbf y)^2
}
\pm \epsilon,
\end{equation}
where \(\lambda_\ell\) and \(\mathbf v_\ell\) are the eigenvalues and
eigenvectors of \(H^\infty\).
\end{theorem}

\begin{proof}
Apply Lemma~\ref{lem:ref-training} to the graph-aggregated inputs
\(\{\tilde X_i\}_{i=1}^N\) using the correspondence in
Section~\ref{subsec:graph-instantiation}. Assumption~\ref{assum-e-2}
ensures that the required normalization, bounded-label, non-degeneracy,
initialization, fixed-second-layer, and gradient descent conditions hold.
Substituting \(n=N\), \(\mathbf s=\mathbf y\), and
\(H^{\mathrm{ref}}=H^\infty\) gives Eq.~\eqref{eq:graph-training-bound}.
\end{proof}

\subsection{Proof of the Test Error Bound}
\label{subsec:proof-test}

\begin{theorem}[Test error bound]
\label{test_theo_fo}
Fix \(\delta\in(0,1)\). Under Assumption~\ref{assum-e-2}, let
\[
\kappa
=
O\!\left(\frac{\lambda_0\delta}{N}\right),
\qquad
m
\ge
\kappa^{-2}\cdot
\mathrm{poly}(N,\lambda_0^{-1},\delta^{-1}).
\]
For
\[
t
\ge
\Omega\!\left(
\frac{1}{\eta\lambda_0}
\log\frac{N}{\delta}
\right),
\]
the GNN trained by gradient descent satisfies, with probability at least
\(1-\delta\),
\begin{equation}
\label{eq:graph-test-bound}
L_{\mathcal D_G}(W_t)
\le
\sqrt{\mathrm{GKC}(H^\infty,\mathbf y)}
+
O\!\left(
\sqrt{
\frac{\log\frac{N}{\lambda_0\delta}}{N}
}
\right).
\end{equation}
\end{theorem}

\begin{proof}
Apply Lemma~\ref{lem:ref-test} to the graph-aggregated inputs
\(\{\tilde X_i\}_{i=1}^N\). Using the correspondence in
Section~\ref{subsec:graph-instantiation}, we obtain
\begin{equation}
L_{\mathcal D_G}(W_t)
\le
\sqrt{
\frac{2\mathbf y^\top (H^\infty)^{-1}\mathbf y}{N}
}
+
O\!\left(
\sqrt{
\frac{\log\frac{N}{\lambda_0\delta}}{N}
}
\right).
\end{equation}
By the definition of GKC in Eq.~\eqref{eq:gkc-appendix},
\[
\frac{2\mathbf y^\top (H^\infty)^{-1}\mathbf y}{N}
=
\mathrm{GKC}(H^\infty,\mathbf y),
\]
which yields Eq.~\eqref{eq:graph-test-bound}.
\end{proof}

\subsection{Proof of the Edge-Specific KC-Score Bound}
\label{subsec:proof-edge-bound}

For an edge \(e_{ij}\), let \(G_{-(i,j)}\) denote the graph after
removing \(e_{ij}\), and let \(H^\infty_{-(i,j)}\) denote the
corresponding graph kernel Gram matrix. The KC score is defined as
\begin{equation}
\label{eq:kc-score-appendix}
KC(i,j)
=
\left|
\mathrm{GKC}(H^\infty_{-(i,j)},\mathbf y)
-
\mathrm{GKC}(H^\infty,\mathbf y)
\right|.
\end{equation}

\begin{corollary}[Edge-specific test error bound]
\label{edge_theo_fo}
Fix \(\delta\in(0,1)\). Under Assumption~\ref{assum-e-2}, let
\[
\kappa
=
O\!\left(\frac{\lambda_0\delta}{N}\right),
\qquad
m
\ge
\kappa^{-2}\cdot
\mathrm{poly}(N,\lambda_0^{-1},\delta^{-1}).
\]
For
\[
t
\ge
\Omega\!\left(
\frac{1}{\eta\lambda_0}
\log\frac{N}{\delta}
\right),
\]
the GNN trained on \(G_{-(i,j)}\) satisfies, with probability at least
\(1-\delta\),
\begin{equation}
\label{eq:edge-bound}
L_{\mathcal D_{G_{-(i,j)}}}(W_t)
\le
\sqrt{\mathrm{GKC}(H^\infty,\mathbf y)}
+
\sqrt{KC(i,j)}
+
O\!\left(
\sqrt{
\frac{\log\frac{N}{\lambda_0\delta}}{N}
}
\right).
\end{equation}
\end{corollary}

\begin{proof}
By the edge-robust non-degeneracy condition in
Assumption~\ref{assum-e-2}, the edge-deleted graph \(G_{-(i,j)}\)
satisfies
\[
\lambda_{\min}(H^\infty_{-(i,j)})\ge \lambda_0>0.
\]
Therefore, Theorem~\ref{test_theo_fo} applies to \(G_{-(i,j)}\), giving
\begin{equation}
\label{eq:edge-deleted-test}
L_{\mathcal D_{G_{-(i,j)}}}(W_t)
\le
\sqrt{\mathrm{GKC}(H^\infty_{-(i,j)},\mathbf y)}
+
O\!\left(
\sqrt{
\frac{\log\frac{N}{\lambda_0\delta}}{N}
}
\right).
\end{equation}

Let
\[
a=\mathrm{GKC}(H^\infty_{-(i,j)},\mathbf y),
\qquad
b=\mathrm{GKC}(H^\infty,\mathbf y),
\qquad
c=KC(i,j)=|a-b|.
\]
Since \(a\le b+c\), we have
\[
\sqrt a \le \sqrt{b+c} \le \sqrt b+\sqrt c .
\]
Substituting this inequality into Eq.~\eqref{eq:edge-deleted-test}
yields Eq.~\eqref{eq:edge-bound}.
\end{proof}

\newpage

\section*{GenAI Usage Disclosure}
Generative AI tools were used for language polishing, grammar correction, and assistance with code/layout refinement. All AI-assisted text, code, and figures were reviewed, verified, and revised by the authors, who take full responsibility for the final content.

\printbibliography

\end{document}